\documentclass{article}

\usepackage[final]{neurips_2018}

     \usepackage[final]{neurips_2018}

\usepackage{url}            
\usepackage{booktabs}       
\usepackage{amsfonts}       
\usepackage{nicefrac}       
\usepackage{microtype}      
\usepackage{amsmath,amssymb}
\usepackage{paralist}
\usepackage{tikz}
\usetikzlibrary{positioning,fit}
\usepackage{graphicx}
\usepackage{caption}
\usepackage{subcaption}
\usepackage{comment}
\usepackage{pifont}

\usepackage{amsthm}
\usepackage[colorinlistoftodos]{todonotes}
\usepackage{hyperref}
\usepackage{listings}
\usepackage{natbib}

\usepackage{mythm}

\AtBeginDocument{

}

\newtheorem{conjecture}[theorem]{Conjecture}

\DeclareMathOperator*{\argmax}{\arg\!\max}
\DeclareMathOperator*{\argmin}{\arg\!\min}

\newcommand{\pol}{{\pi}}
\newcommand{\Pol}{{\Pi}}

\newcommand{\pH}{{\dot{\pol}}}
\newcommand{\sstart}{{\hat{s}}}

\newcommand{\St}{{\mathcal{S}}}
\newcommand{\Ac}{{\mathcal{A}}}
\newcommand{\Rw}{{\mathcal{R}}}
\newcommand{\Pl}{{\mathcal{P}}}

\newcommand{\RH}{{\dot{R}}}
\newcommand{\plH}{{\dot{p}}}

\DeclareMathOperator{\Reg}{Reg}

\title{Occam's razor is insufficient to infer the preferences of irrational agents}

\author{
Stuart Armstrong \thanks{Equal contribution.}{*} \thanks{Further affiliation: Machine Intelligence Research Institute, Berkeley, USA.} \\
Future of Humanity Institute \\
University of Oxford \\
\texttt{stuart.armstrong@philosophy.ox.ac.uk} \\
\And
S{\"o}ren Mindermann* \thanks{Work performed at Future of Humanity Institute.} \\
Vector Institute \\
University of Toronto \\
\texttt{soeren.mindermann@gmail.com} \\
}

\begin{document}
\maketitle

\begin{abstract}
Inverse reinforcement learning (IRL) attempts to infer human rewards or preferences from observed behavior. Since human planning systematically deviates from rationality, several approaches have been tried to account for specific human shortcomings. 
However, the general problem of inferring the reward function of an agent of unknown rationality has received little attention.
Unlike the well-known ambiguity problems in IRL, this one is practically relevant but cannot be resolved by observing the agent's policy in enough environments.
This paper shows (1) that a No Free Lunch result implies it is impossible to uniquely decompose a policy into a planning algorithm and reward function, and (2) that even with a reasonable simplicity prior/Occam's razor on the set of decompositions, we cannot distinguish between the true decomposition and others that lead to high regret.
To address this, we need simple `normative' assumptions, which cannot be deduced exclusively from observations.
\end{abstract}

\renewcommand{\tableautorefname}{Code} 
\renewcommand{\tablename}{Code} 


\section{Introduction}
In today's reinforcement learning systems, a simple reward function is often hand-crafted, and still sometimes leads to undesired behaviors on the part of RL agent, as the reward function is not well aligned with the operator's true goals\footnote{See for example the game CoastRunners, where an RL agent didn't finish the course, but instead found a bug allowing it to get a high score by crashing round in circles \url{https://blog.openai.com/faulty-reward-functions/}.}. As AI systems become more powerful and autonomous, these failures will become more frequent and grave as RL agents exceed human performance, operate at time-scales that forbid constant oversight, and are given increasingly complex tasks --- from driving cars to planning cities to eventually evaluating policies or helping run companies. Ensuring that the agents behave in alignment with human values is known, appropriately, as the \textit{value alignment problem} \citep{Amodei,Hadfield-Menell,Russell2015,superI,leike2017ai}.

One way of resolving this problem is to infer the correct reward function by observing human behaviour.
This is known as Inverse reinforcement learning (IRL) \citep{Ng2000,Abbeel2004,ziebart}. Often, learning a reward function is preferred over imitating a policy: when the agent must outperform humans, transfer to new environments, or be interpretable. The reward function is also usually a (much) more succinct and robust task representation than the policy, especially in planning tasks \citep{Abbeel2004}. Moreover, supervised learning of long-range and goal-directed behavior is often difficult without the reward function \citep{Ratliff2006}.

Usually, the reward function is inferred based on the assumption that human behavior is optimal or noisily optimal. However, it is well-known that humans deviate from rationality in \textit{systematic}, non-random ways \citep{Tversky1975}. This can be due to specific biases such as time-inconsistency, loss aversion and hundreds of others, but also limited cognitive capacity, which leads to forgetfulness, limited planning and false beliefs.
This limits the use of IRL methods for tasks that humans don't find trivial.

Some IRL approaches address specific biases \citep{Evans, Evans2016}, and others assume noisy rationality \citep{ziebart, Boularias2011}. 
But a general framework for inferring the reward function from suboptimal behavior does not exist to our knowledge. 
Such a framework needs to infer two unobserved variables simultaneously: the human reward function and their planning algorithm\footnote{
Technically we only need to infer the human reward function, but inferring that from behaviour requires some knowledge of the planning algorithm.
} which connects the reward function with behaviour, henceforth called a \emph{planner}.

The task of observing human behaviour (or the human policy) and inferring from it the human reward function and planner will be termed \emph{decomposing} the human policy.
This paper will show there is a No Free Lunch theorem in this area: it is impossible to get a unique decomposition of human policy and hence get a unique human reward function.
Indeed, \emph{any} reward function is possible.
And hence, if an IRL agent acts on what it believes is the human policy, the potential regret is near-maximal.
This is another form of unidentifiability of the reward function, beyond the well-known ones \citep{Ng2000,Amin2016}.

The main result of this paper is that, unlike other No Free Lunch theorems, this unidentifiability does not disappear when regularising with a general simplicity prior that formalizes Occam's razor \citep{vitanyi1997introduction}.
This result will be shown in two steps: first, that the simplest decompositions include degenerate ones, and secondly, that the most `reasonable' decompositions according to human judgement are of high complexity.

So, although current IRL methods can perform well on many well-specified problems, they are fundamentally and philosophically incapable of establishing a `reasonable' reward function for the human, no matter how powerful they become.
In order to do this, they will need to build in `normative assumptions':  key assumptions about the reward function and/or planner, that cannot be deduced from observations, and allow the algorithm to focus on good ways of decomposing the human policy.

Future work will sketch out some potential normative assumptions that can be used in this area, making use of the fact that humans assess each other to be irrational, and often these assessments agree.
In view of the No Free Lunch result, this shows that humans must share normative assumptions.

One of these `normative assumption' approaches is briefly illustrated in an appendix, while another appendix demonstrates how to use the planner-reward formalism to define when an agent might be manipulating or overriding human preferences.
This happens when the agent pushes the human towards situations where their policy is very suboptimal according to their reward function.

\section{Related Work}

In the first IRL papers from \citet{Ng2000} and \citet{Abbeel2004} a max-margin algorithm was used to find the reward function under which the observed policy most outperforms other policies. Suboptimal behavior was first addressed explicitly by \citet{Ratliff2006} who added slack variables to allow for suboptimal behavior. This finds reward functions such that the observed policy outperforms most other policies and the biggest margin by which another policy outperforms it is minimal, i.e. the observed policy has low regret. \citet{Shiarlis} introduce a modern max-margin technique with an approximate planner in the optimisation.

However, the max-margin approach has mostly been replaced by the max entropy IRL \citep{ziebart}. Here, the assumption is that observed actions or trajectories are chosen with probability proportional to the exponent of their value. This assumes a specific suboptimal planning algorithm which is \textit{noisily} rational (also known as \textit{Boltzmann}-rational). Noisy rationality explains human behavior on various data sets better \citep{Hula2015}. However, \citet{Evans} and \citet{Evans2016} showed that this can fail since humans deviate from rationality in systematic, non-random ways. If noisy rationality is assumed, repeated suboptimal actions throw off the inference.


Literature on inferring the reasoning capabilities of an agent is scarce. \citet{Evans} and \citet{Evans2016} use Bayesian inference to identify specific planning biases such as myopic planning and hyperbolic time-discounting. They simultaneously infer the agent's preferences. \citet{cundy2018exploring} adds bias resulting from hierarchical planning.
\citet{Hula2015} similarly let agents infer features of their opponent's reasoning such as planning depth and impulsivity in simple economic games. Recent work learns the planning algorithm with two assumptions: being close to noisily rational in a high-dimensional planner space and supervised planner-learning \citep{anonymous2019inferring}.

The related ideas of meta-reasoning \citep{Russell2016}, computational rationality \citep{Lewis2014} and resource rationality \citep{Griffiths2015} may create the possibility to redefine irrational behavior as rational in an `ancestral' distribution of environments where the agent optimises its rewards by choosing among the limited computations it is able to perform or jointly minimising the cost of computation and maximising reward. This could in theory redefine many biases as computationally optimal in some distribution of environments and provide priors on human planning algorithms. Unfortunately the problem of doing this in practice seems to be extremely difficult --- and it assumes that human goals are roughly the same as evolution's goals, which is certainly not the case.

\section{Problem setup and background}
A human will be performing a series of actions, and from these, an agent will attempt to estimate both the human's reward function and their planning algorithm.

The environment $M$ in which the human operates is an MDP/R, a Markov Decision Process without reward function (a \emph{world-model} \citep{hadfield2017inverse}).
An MDP/R is defined as a tuple, $\langle \St,\Ac,T,\sstart \rangle$ consisting of a discrete state space $\St$, a finite action space $\Ac$, a fixed starting state $\sstart$, and a probabilistic transition function $T:\St\times \Ac \times \St \to [0,1]$ to the next state (also called the \textit{dynamics}).
At each step, the human is in a certain state $s$, takes a certain action $a$, and ends up in a new state $s'$ as given by $T(s'\mid s,a)$.

Let $\Rw = \{R: \St\times\Ac \to [-1,1] \} = [-1,1]^{\St \times \Ac}$ be the space of candidate reward functions; a given $R$ will map any state-reward pair to a reward value in the interval $[-1,1]$.

Let $\Pol$ be the space of deterministic, Markovian
policies.
So $\Pol$ is the space of functions $\St\to\Ac$.
The human will be following the policy $\pH\in\Pol$.

The results of this paper apply to both discounted rewards and episodic environments settings\footnote{The setting is only chosen for notational convenience: it also emulates discrete POMDPs, non-Markovianness (eg by encoding the whole history in the state) and pseudo-random policies. 
}.

\subsection{Planners and reward functions: decomposing the policy}

The human has their reward function, and then follows a policy that presumably attempts to maximise it.
Therefore there is something that bridges between the reward function and the policy: a piece of greater or lesser rationality that transforms knowledge of the reward function into a plan of action.

This bridge will be modeled as a \emph{planner} $p: \Rw \to \Pol$, a function that takes a reward and outputs a policy.
This planner encodes all the rationality, irrationality, and biases of the human.
Let $\Pl$ be the set of planners.
The human is therefore defined by a \emph{planner-reward} pair $(p, R)\in\Pl\times\Rw$. Similarly, $(p,R)$ with $p(R)=\pol$ is a \emph{decomposition} of the policy $\pol$.
The task of the agent is to find a `good' decomposition of the human policy $\pH$.

\subsection{Compatible pairs and evidence}

The agent can observe the human's behaviour and infer their policy from that.
In order to simplify the problem and separate out the effect of the agent's learning, we will assume the agent has perfect knowledge of the human policy $\pH$ and of the environment $M$.
At this point, the agent cannot learn anything by observing the human's actions, as it can already perfectly predict these.

Then a pair $(p, R)$ is defined to be \emph{compatible} with $\pH$, if $p(R)=\pH$ --- thus that pair is a possible candidate for decomposing the human policy into the human's planner and reward function.




\section{Irrationality-based unidentifiability}

Unidentifiability of the reward is a well-known problem in IRL \citep{Ng2000}. \citet{Amin2016} categorise the problem into \textit{representational} and \textit{experimental} unidentifiability. The former means that adding a constant to a reward function or multiplying it with a positive scalar does not change what is optimal behavior.
This is unproblematic as rescaling the reward function doesn't change the preference ordering.
The latter can be resolved by observing 
optimal policies in a whole class of MDPs which contains all possible transition dynamics.
We complete this framework with a third kind of identifiability, which arises when we observe suboptimal agents.
This kind of unidentifiability is worse as it cannot necessarily be resolved by observing the agent in many tasks. In fact, it can lead to almost arbitrary regret.

\subsection{Weak No Free Lunch: unidentifiable reward function and half-maximal regret}
The results in this section show that without assumptions about the rationality of the human, all attempts to optimise their reward function are essentially futile. \citet{Everitt2017} work in a similar setting as we do: in their case, a corrupted version of the reward function is observed.
The problem our case is that a `corrupted' version $\pH$ of an optimal policy $\pol^*_{\dot{R}}$ is observed and used as information to optimise for the ideal reward $\dot{R}$. A No Free Lunch result analogous to theirs applies in our case; both resemble the No Free Lunch theorems for optimisation \citep{Wolpert1997}.

More philosophically, this result is as an instance of the well-known \textit{is-ought} problem from meta-ethics. \citet{hume1888treatise} argued that what \emph{ought} to be (here, the human's reward function) can never be concluded from what \emph{is} (here, behavior) without extra assumptions. Equivalently, the human reward function cannot be inferred from behavior without assumptions about the planning algorithm $p$. In probabilistic terms, the likelihood 
$P(\pol|R)=\sum_{p\in\Pl}P(\pol\mid R,p)P(p)$ is undefined without $P(p)$.
As shown in \autoref{no:simplicity} and \autoref{comp:human:reward}, even a simplicity prior on $p$ and $R$ will not help.

\subsubsection{Unidentifiable reward functions}
Firstly, we note that compatibility ($p(R)=\pH$), puts no restriction on $R$, and few restrictions on $p$:
\begin{theorem}\label{theo:weak:nfl}
For all $\pol\in\Pol$ and $R\in\Rw$, there exists a $p\in\Pl$ such that $p(R)=\pol$.

For all $p\in\Pl$ and $\pol\in\Pol$ in the image of $p$, there exists an $R$ such that $p(R)=\pol$.
\end{theorem}
\begin{proof}
Trivial proof: define the planner\footnote{This is the `indifferent' planner $p_\pol$ of \autoref{degen:pair}.} $p$ as mapping all of $\Rw$ to $\pol$; then $p(R)=\pol$.
The second statement is even more trivial, as $\pol$ is in the image of $p$, so there must exist $R$ with $p(R)=\pol$.
\end{proof}

\subsubsection{Half-maximal regret}

The above shows that the reward function cannot be constrained by observation of the human, but what about the expected long-term value?
Suppose that an agent is unsure what the actual human reward function is; if the agent itself is acting in an MDP/R, can it follow a policy that minimises the possible downside of its ignorance?

This is prevented by a recent No Free Lunch theorem.
Being ignorant of the reward function one should maximise is equivalent of having a \emph{corrupted reward channel} with arbitrary corruption.
In that case, \citet{Everitt2017} demonstrated that whatever policy $\pol$ the agent follows, there is a $R\in\Rw$ for which $\pol$ is half as bad as the worst policy the agent could have followed. Specifically, let $V^\pol_R(s)$ be the expected return of reward function $R$ from state $s$, given that the agent follows policy $\pol$.
If $\pol$ was the optimal policy for $R$, then this can be written as $V^*_R(s)$.
The regret of $\pol$ for $R$ at $s$ is given by the difference:
\begin{align*}
    \Reg(\pol,R)(s) = V^*_R(s)-V^\pol_R(s).
\end{align*}
Then \citet{Everitt2017} demonstrates that for any $\pol$,
\begin{align*}
    \max_{R\in\Rw} \Reg(\pol,R)(s) \geq \frac{1}{2} \left(\max_{\pol' \in \Pol,R \in \Rw} \Reg(\pol',R)(s)\right).
\end{align*}
So for any compatible $(p,R)=\pH$, we cannot rule out that maximizing $R$ leads to at least half of the worst-case regret.

\section{Simplicity of degenerate decompositions}\label{no:simplicity}

Like many No Free Lunch theorems, the result of the previous section is not surprising given there are no assumptions about the planning algorithm.
No Free Lunch results are generally avoided by placing a simplicity prior on the algorithm, dataset, function class or other object \citep{Everitt}.
This amounts to saying algorithms can benefit from regularisation.
This section is dedicated to showing that, surprisingly, simplicity does not solve the No Free Lunch result.

Our simplicity measure is minimum description length of an object, defined as Kolmogorov complexity \citep{kolmogorov1965three}, the length of the shortest program that outputs a string describing the object. This is the most general formalization of Occam's razor we know of \citep{vitanyi1997introduction}.
\autoref{other:alg:com} explores how the results extend to other measures of complexity, such as those that include computation time. We start with informal versions of our main results.

\begin{theorem}[Informal simplicity theorem]\label{inf:sim}
Let $(\plH,\RH)$ be a `reasonable' planner-reward pair that captures our judgements about the biases and rationality of a human with policy $\pH = \plH(\RH)$. Then there are degenerate planner-reward pairs, compatible with $\pH$, of lower complexity than $(\plH,\RH)$, and a pair $(\plH',-\RH)$ of similar complexity to $(\plH,\RH)$, but with opposite reward function.
\end{theorem}

There are a few issues with this theorem as it stands.
Firstly, simplicity in algorithmic information theory is relative to the computer language (or equivalently Universal Turing Machine) $L$ used \citep{ming2014kolmogorov, Calude2002}, and there exists languages in which the theorem is clearly false: one could choose a degenerate language in which $(\plH,\RH)$ is encoded by the string `$0$', for example, and all other planner-reward pairs are of extremely long length.
What constitutes a `reasonable' language is a long-standing open problem, see \citet{leike2017ai} and \citet{muller2010stationary}.
For any pair of languages, complexities differ only by a constant, the amount required for one language to describe the other, but this constant can be arbitrarily large.

Nevertheless, this section will provide grounds for the following two semi-formal results:
\begin{proposition}\label{degen:prop}
If $\pH$ is a human policy, and $L$ is a `reasonable' computer language, then there exists degenerate planner-reward pairs amongst the pairs of lowest complexity compatible with $\pH$.
\end{proposition}

\begin{proposition}\label{minus:prop}
If $\pH$ is a human policy, and $L$ is a `reasonable' computer language with $(\plH,\RH)$ a compatible planner-reward pair, then there exist a pair $(\plH',-\RH)$ of comparable complexity to $(\plH,\RH)$, but opposite reward function.
\end{proposition}

The last part of \autoref{inf:sim}, the fact that any `reasonable' $(\plH,\RH)$ is expected to be of higher complexity, will be addressed in 
\autoref{no:complex_human}.
%

\subsection{Simple degenerate pairs}

The argument in this subsection will be that 1) the complexity of $\pH$ is close to a lower bound on any pair compatible with it and 2) degenerate decompositions are themselves close to this bound. The first statement follows because for any decomposition $(p,R)$ compatible with $\pH$, the map $(p,R) \mapsto p(R)=\pH$ will be a simple one, adding little complexity.
And if a compatible pair $(p',R')$ can be from $\pH$ with little extra complexity, then it too will have a complexity close to the minimal complexity of any other pair compatible with it. Therefore we will first produce three degenerate pairs that can be simply constructed from $\pH$.

\subsubsection{The degenerate pairs}\label{degen:pair}

We can define the trivial constant reward function $0$, and the greedy planner $p_g$.
The greedy planner $p_g$ acts by taking the action that maximises the immediate reward in the current state and the next action.
Thus\footnote{Recall that $p_g$ is a planner, $p_g(R)$ is a policy, so $p_g(R)$ can be applied to states, and $p_g(R)(s)$ is an action.} $p_g(R)(s)=\argmax_a R(s,a)$.
We can also define the anti-greedy planner $-p_g$, with $-p_g(R)(s)=\argmin_a R(s,a)$.
In general, it will be useful to define the negative of a planner:
\begin{definition}\label{negation:definition}
If $p:\Rw\to \Pol$ is a planner, the planner $-p$ is defined by $-p(R)=p(-R)$.
\end{definition}
For any given policy $\pol$, we can define the \emph{indifferent} planner $p_\pol$, which maps any reward function to $\pol$.
We can also define the reward function $R_\pol$, so that $R_\pol(s,a)=1$ if $\pol(s)=a$, and $R_\pol(s,a)=0$ otherwise.
The reward function $-R_\pol$ is defined to be the negative of $R_\pol$. Then:
\begin{lemma}\label{compatible:lemma}
The pairs $(p_\pol,0)$, $(p_g, R_\pol)$, and $(-p_g,-R_\pol)$ are all compatible with $\pol$.
\end{lemma}
\begin{proof}
Since the image $p_\pol$ is $\pol$, $p_\pol(0)=\pol$.
Now, $R_\pol(s,a)>0$ iff $\pol(s)=a$, hence for all $s$:
\begin{align*}
p_g(R_\pol)(s)=\argmax_a R_\pol(s,a) = \pol(s),
\end{align*} so $p_g(R_\pol) = \pol$. Then $-p_g(-R_\pol)=p_g(-(-R_\pol))=p_g(R_\pol)=\pol$, by \autoref{negation:definition}.
\end{proof}

\subsubsection{Complexity of basic operations}\label{basic:op}

We will look the operations that build the degenerate planner-reward pairs from any compatible pair:
\begin{enumerate}
    \item For any planner $p$, $f_1(p)=(p,0)$ as a planner-reward pair.
    \item For any reward function $R$, $f_2(R)=(p_g,R)$.
    \item For any planner-reward pair $(p,R)$, $f_3(p,R)=p(R)$.
    \item For any planner-reward pair $(p,R)$, $f_4(p,R)=(-p,-R)$.
    \item For any policy $\pol$, $f_5(\pol)=p_\pol$.
    \item For any policy $\pol$, $f_6(\pol)=R_\pol$.
\end{enumerate}
These will be called the basic operations, and there are strong arguments that reasonable computer languages should be able to express them with short programs. The operation $f_1$, for instance, is simply appending the flat trivial $0$, $f_2$ appends a planner defined by the simple\footnote{
In most standard computer languages, $\argmax$ just requires a $\mathrm{for}$-loop, a reference to $R$, a comparison with a previously stored value, and possibly the storage of a new value and the current action.
} search operator $\argmax$, $f_3$ applies a planner to the object --- a reward function --- that the planner naturally acts on, $f_4$ is a double negation, while $f_5$ and $f_6$ are simply described in \autoref{degen:pair}.

From these basic operations, we can define three composite operations that map any compatible planner-reward pair to one of the degenerate pairs (the element $F_4=f_4$ is useful for later definitions). Thus define
\begin{align*}
    F = \{F_1 = f_1\circ f_5 \circ f_3, \ \ F_2 = f_2 \circ f_6 \circ f_3,\ \ F_3 = f_4 \circ f_2 \circ f_6 \circ f_3, \ \ F_4 = f_4 \}.
\end{align*}
For any $\pH$-compatible pair $(p,R)$ we have $F_1(p,R)=(p_\pH,0)$, $F_2(p,R)=(p_g,R_\pH)$, and $F_3(p,R)=(-p_g,-R_\pH)$ (see the proof of \autoref{degen:prop:formal}).

Let $K_L$ denote Kolmogorov complexity in the language L: the shortest algorithm in $L$ that generates a particular object. We define the $F$-complexity of $L$ as
\begin{align*}
\max_{(p,R),F_i\in F} K_L(F_i(p,R)) - K_L(p,R).
\end{align*}
Thus the $F$-complexity of $L$ is how much the $F_i$ potentially increase\footnote{$F$-complexity is non-negative: $F_4 \circ F_4$ is the identity, so that $K_L (F_4(p,R)) - K_L(p,R) = - (K_L(F_4(F_4(p,R))-K_L(F_4(p,R))$, meaning that $\max_{(p,R),F_4} K_L (F_4(p,R)) - K_F(p,R)$ must be non-negative; this is a reason to include $F_4$ in the definition of $F$.}
the complexity of pairs.

For a constant $c \geq 0$, this allows us to formalise what we mean by $L$ being a $c$-reasonable language for $F$: that the $F$-complexity of $L$ is at most $c$.
A reasonable language is a $c$-reasonable language for a $c$ that we feel is intuitively low enough.

\subsubsection{Low complexity of degenerate planner-reward pairs}
To formalise the concepts `of lowest complexity', and `of comparable complexity', choose a constant $c \geq 0$, then $(p,R)$ and $(p',R')$ are of `comparable complexity' if
\begin{align*}
||K_L(p,R) - K_L(p',R')|| \leq c.
\end{align*}
For a set $S\subset \Pl\times\Rw$, the pair $(p,R)\in S$ is amongst the lowest complexity in $S$ if
\begin{align*}
||K_L(p,R) - \min_{(p',R')\in S} K_L(p',R')|| \leq c,
\end{align*}
thus $K_L$ is within distance $c$ of the minimum complexity element of $S$.
Now formalize \autoref{degen:prop}:
\begin{proposition}\label{degen:prop:formal}
If $\pH$ is the human policy, $c$ defines a reasonable measure of comparable complexity, and $L$ is a $c$-reasonable language for $F$, then the degenerate planner-reward pairs $(p_\pH,0)$, $(p_g, R_\pH)$, and $(-p_g,-R_\pH)$ are amongst the pairs of lowest complexity among the pairs compatible with $\pH$.
\end{proposition}
\begin{proof}
By \autoref{compatible:lemma}, $(p_\pH,0)$, $(p_g, R_\pH)$, and $(-p_g,-R_\pH)$ are compatible with $\pH$. By the definitions of the $f_i$ and $F_i$, for$(p,R)$ compatible with $\pH$, $f_3((p, R))=p(R)=\pH$ and hence
\begin{align*}
F_1(p,R) &= f_1 \circ f_5 (\pH) = f_1 (p_\pH) = (p_\pH,0),\\
F_2(p,R) &= f_2 \circ f_6 (\pH) = f_2 (R_\pH) = (p_g,R_\pH),\\
F_3(p,R) &= f_4 \circ F_2(p,R) = (-p_g,-R_\pH).
\end{align*}
Now pick $(p,R)$ to be the simplest pair compatible with $\pH$.
Since $L$ is $c$-reasonable for $F$, $K_L(p_\pH,0) \leq c + K_L(p,R)$.
Hence $(p_\pH,0)$ is of lowest complexity among the pairs compatible with $\pH$; the same argument applies for the other two degenerate pairs.
\end{proof}

\subsection{Negative reward}\label{comp:human:reward}

If $(\plH,\RH)$ is compatible with $\pH$, then so is $(-\plH,-\RH)=f_4(\plH,\RH)=F_4(\plH,\RH)$.
This immediately implies the formalisation of \autoref{minus:prop}:
\begin{proposition}
If $\pH$ is a human policy, $c$ defines a reasonable measure of comparable complexity, $L$ is a $c$-reasonable language for $F$, and $(\plH,\RH)$ is compatible with $\pH$, then $(-\plH,-\RH)$ is of comparable complexity to $(\plH,\RH)$.
\end{proposition}
So complexity fails to distinguish between a reasonable human reward function and its negative.

\section{The high complexity of the genuine human reward function}
\label{no:complex_human}

\autoref{no:simplicity} demonstrated that there are degenerate planner-reward pairs close to the minimum complexity among all pairs compatible with $\pH$.
This section will argue that any reasonable pair $(\plH,\RH)$ is unlikely to be close to this minimum, and is therefore of higher complexity than the degenerate pairs. Unlike simplicity, reasonable decomposition cannot easily be formalised. Indeed, a formalization would likely already solve the problem, yielding an algorithm to maximize it.
Therefore, the arguments in this section are mostly qualitative. 

We use reasonable to mean `compatible with human judgements about rationality'.
Since we do not have direct access to such a decomposition, the complexity argument will be about showing the complexity of these human judgements. This argument will proceed in three stages:
\begin{enumerate}
\item Any reasonable $(\plH,\RH)$ is of high complexity, higher than it may intuitively seem to us.
\item Even given $\pH$, any reasonable $(\plH,\RH)$ involves a high number of contingent choices. Hence any given $(\plH,\RH)$ has high information (and thus high complexity), even given $\pH$.
\item Past failures to find a simple $(\plH,\RH)$ derived from $\pH$ are evidence that this is tricky.
\end{enumerate}

\subsection{The complexity of human (ir)rationality}
Humans make noisy and biased decisions all the time. Though noise is important \citep{kahneman2016noise}, many biases, such as anchoring bias, overconfidence, planning fallacies, and so on, affect humans in a highly systematic way; see \citet{kahneman2011thinking} for many examples.

Many people may feel that they have a good understanding of rationality, and therefore assume that assessing the (ir)rationality of any particular decision is not a complicated process.
But an intuition for bias does not translate into a process for establishing a $(\plH,\RH)$.

Consider the anchoring bias defined in \citet{ariely2004arbitrarily}, where irrelevant information --- the last digits of social security numbers --- changed how much people were willing to pay for goods.
When defining a reasonable $(\plH,\RH)$, it does not suffice to be aware of the existence of anchoring bias\footnote{
The fact that many cognitive biases have only been discovered recently argue against people having a good intuitive grasp of bias and rationality, as do people's persistent bias blind spots \citep{scopelliti2015bias}.
}, but one has to precisely quantify the extent of the bias --- why does anchoring bias seem to be stronger for chocolate than for wine, for instance?
And why these precise percentages and correlations, and not others?
And can people's judgment tell which people are more or less susceptible to anchoring bias?
And can one quantify the bias for a single individual, rather than over a sample?

Any given $(\plH,\RH)$ can quantify the form and extent of these biases by computing objects like the regret function $\Reg(\plH,\RH)(s) := \Reg(\plH(\RH),\RH)(s) = V^*_{\RH}(s)-V^{\plH(\RH)}_{\RH}(s)$, which measures the divergence between the expected value of the actual and optimal human policies\footnote{To exactly quantify the anchoring bias above, we could use a regret function that contrasts $\pH$ with the same policy, but where the decision is optimal for one turn only (rather than for all turns, as in standard regret).}.
Thus any given $(\plH,\RH)$ --- which contains the information to compute quantities like $\Reg(\plH,\RH)(s)$ or similar measures of bias\footnote{In constrast, regret for the degenerate planner-reward pairs is trivial. $\Reg(p_\pH,0)$ and $\Reg(p_g,R_\pH)$ are identically zero --- in the second case, since $p_g(R_\pH)$ is actually optimal for $R_\pH$, getting the maximal possible reward --- while $(-p_g,-R_\pH)$ has a regret that is identically $-1$ at each step.
}, in every state --- carries a high amount of numerical information about bias, and hence a high complexity.

Since humans do not easily have access to this information, this implies that human judgement of irrationality is subject to Moravec's paradox \citep{moravec1988mind}.
It is similar to, for example, social skills: though it seems intuitively simple to us, it is highly complex to define in algorithmic terms.

Other authors have argued directly for the complexity of human values, from fields as diverse as computer science, philosophy, neuroscience, and economics
\citep{minsk_AI_risk,superI,glimcher2009neuroeconomics,muehlhauser2012singularity,yudkowsky2011complex}.

\subsection{The contingency of human judgement}
The previous section showed that reasonable $(\plH,\RH)$ carry large amounts of information/complexity, but the key question is whether it requires information \textit{additional} to that in $\pH$.
This section will show that
even when $\pH$ is known, there are many contingent choices that need to be made to define any specific reasonable $(\plH,\RH)$.
Hence any given $(\plH,\RH)$ contains a large amount of information beyond that in $\pH$, and hence is of higher complexity.

Reasons to believe that human judgement about reasonable $(\plH,\RH)$ contains many contingent choices:
\begin{itemize}
    \item There is a variability of human judgement between cultures. When \citet{miller1984culture} compared American and Indian assessments of the same behaviours, they found systematically different explanations for them\footnote{``Results show that there were cross-cultural and developmental differences related to contrasting cultural conceptions of the person [...] rather than from cognitive, experiential, and informational differences [...].''}
    Basic intuitions about rationality also vary between cultures \citep{nisbett2001culture,bruck1999ritual}.
    \item There is a variability of human judgement within a single culture. When \citet{slovic1974accepts} analysed the ``Allais Paradox'', they found that different people gave different answers as to what the rational behaviour was in their experiments.
    \item There is evidence of variability of human judgement within the same person. \citet{slovic1974accepts} further attempted to argue for the rationality of one of the answers.
    This sometimes resulted in the participant sometimes changing their minds, and contradicting their previous assessment of rationality.
    \item There is a variability of human judgement for the same person assessing their own values, caused by differences as trivial as question ordering \citep{schuman1983norm}.
    So human meta-judgement, of own values and rationality, is also contingent and variable.
    \item People have partial bias blind spots around their own biases \citep{scopelliti2015bias}.
\end{itemize}
Thus if a human is following policy $\pH$, a decomposition $(\plH,\RH)$ would provide additional information about the cultural background of the decomposer, their personality within their culture, and even about the past history of the decomposer and how the issue is being presented to them.
Those last pieces prevents us from `simply' using the human's own assessment of their own rationality, as that assessment is subject to change and re-interpretation depending on their possible histories.

\subsection{The search for human rationality models}

One final argument that there is no simple algorithm for going from $\pH$ to $(\plH,\RH)$: many have tried and failed to find such an algorithm.
Since the subject of human rationality has been a major one for several thousands of years, the ongoing failure is indicative --- though not a proof --- of the difficulties involved.
There have been many suggested philosophical avenues for finding such a reward (such as reflective equilibrium \citep{veil}), but all have been underdefined and disputed.

The economic concept of revealed preferences \citep{samuelson1948consumption} is the most explicit, using the assumption of rational behaviour to derive human preferences.
This is an often acceptable approximation, but can be taken too far: failure to take achieve an achievable goal does not imply that failure was desired.
Even within the confines of economics, it has been criticised by behavioural economics approaches, such as prospect theory \citep{kahneman2013prospect} --- and there are counter-criticisms to these. 

Using machine learning to deduce the intentions and preferences of humans is in its infancy, but we can see non-trivial real-world examples, even in settings as simple as car-driving \citep{lazar2018maximizing}.

Thus to date, neither humans nor machine learning have been able to find simple ways of going from $\pH$ to $(\plH,\RH)$, nor any simple and \emph{explicit} theory for how such a decomposition could be achieved.
This suggests that $(\plH,\RH)$ is a complicated object, even if $\pH$ is known.
In conclusion:
\begin{conjecture}[Informal complexity proposition]\label{reasonable:complex}
If $\pH$ is a human policy, and $L$ is a `reasonable' computer language with $(\plH,\RH)$ a `reasonable' compatible planner-reward pair, then the complexity of $(\plH,\RH)$ is not close to minimal amongst the pairs compatible with $\pH$.
\end{conjecture}

\section{Conclusion}\label{norm:info}

We have shown that some degenerate planner-reward decompositions of a human policy have near-minimal description length and argued that decompositions we would endorse do not.
Hence, under the Kolmogorov-complexity simplicity prior, a formalization of Occam's Razor, the posterior would endorse degenerate solutions.
Previous work has shown that noisy rationality is too strong an assumption as it does not account for bias; we tried the weaker assumption of simplicity, strong enough to avoid typical No Free Lunch results, but it is insufficient here.

This is no reason for despair: there is a large space to explore between these two extremes. Our hope is that with some minimal assumptions about planner and reward we can infer the rest with enough data. Staying close to agnostic is desirable in some settings: for example, a misspecified model of the human reward function can lead to disastrous decisions with high confidence \citep{milli2017should}.
\citet{anonymous2019inferring} makes a promising first try --- a high-dimensional parametric planner is initialized to noisy rationality and then adapts to fit the behavior of a systematically irrational agent.

How can we reconcile our results with the fact that humans routinely make judgments about the preferences and irrationality of others? And, that these judgments are often correlated from human to human? After all, No Free Lunch applies to human as well as artificial agents.
Our result shows that they must be using shared priors, beyond simplicity, that are not learned from observations. We call these \textit{normative assumptions} because they encode beliefs about which reward functions are more likely and what constitutes approximately rational behavior.
Uncovering minimal normative assumptions would be an ideal way to build on this paper; \autoref{alice:algorithm} shows one possible approach.

\section*{Acknowledgments.} We wish to thank Laurent Orseau, Xavier O'Rourke, Jan Leike, Shane Legg, Nick Bostrom, Owain Evans, Jelena Luketina, Tom Everrit, Jessica Taylor, Paul Christiano, Eliezer Yudkowsky, Stuart Russell, Dylan Hadfield-Menell, and Anders Sandberg, Adam Gleave, Rohin Shah, among many others.
This work was supported by the Alexander Tamas programme on AI safety research, the Leverhulme Trust, and the Machine Intelligence Research Institute.

\bibliography{ref,soren_ref}

\appendix
{







\section{Other measures of algorithmic complexity}\label{other:alg:com}

It might be felt that \autoref{degen:prop:formal} depends on using only the Kolmogorov/algorithmic complexity of $L$.
For example, it seems that though the algorithm defining $R_\pol$ in \autoref{degen:pair} is short, the running time of $(p_g,R_\pol)$ might be much longer than other compatible $(p,R)$ pairs.
This is because $p_g$ defines an argmax over actions while $R_\pol(s,a)$ runs $\pol$ on $s$.
Hence applying $p_g$ to $R_\pol$ requires running $\pol(s)$ as many times as $||\Ac||$, which is very inefficient.

We could instead use a measure of complexity that also uses the number of operations required to compute a pair \citep{schmidhuber2002speed}.

For any object $S$, let $\alpha_S$ be an algorithm that generates $S$ as an output.
If $S$ is a function that can be applied to another object $T$, then $\alpha_S(\alpha_T)$ generates $S(T)$ by generating $S$ with $\alpha_S$, whenever $S$ needs to look at $T$, it uses $\alpha_T$ to generate $T$.

For example, if $\alpha$ is an algorithm in the language of $L$, $l(\alpha)$ its length, and $t(\alpha)$ its running time, we could define the time-bounded Kolmogorov complexity,
\begin{align*}
Kt_L(p,R) &= \min_{\alpha_p, \alpha_R} l(\alpha_p) + l(\alpha_R) + \log(t(\alpha_p(\alpha_R))) \\
KT_L(p,R) &= \min_{\alpha_p, \alpha_R} l(\alpha_p) + l(\alpha_R) + t(\alpha_p(\alpha_R)). 
\end{align*}
The $Kt_L$ derives from \citet{levin1984randomness}, while $KT_L$ is closely related to the example in \citet{allender2001worlds}.
Note that instead of $l(\alpha_p) + l(\alpha_R)$ we could consider the length of a single algorithm that generates both $p$ and $R$; however, for the degenerate pairs we are considering, the length of such an algorithm is very close to $l(\alpha_p) + l(\alpha_R)$, as either $\alpha_p$ or $\alpha_R$ would be trivial.

The main result is that neither $Kt_L$ nor $KT_L$ complexity remove the No Free Lunch Theorem.
For the degenerate pair $(p_\pH,0)$, nothing is gained, because its running time is comparable to $\pH$.
For the other two degenerate pairs, consider the situation where a planner takes as input not a reward function $R\in\Rw$, but the source code in $L$ of an algorithm that computes $R$.
In that case, the previous proposition still applies:
\begin{proposition}
The results of \autoref{degen:prop:formal} still apply to $(p_\pH,0)$ if $Kt_L$ or $KT_L$ are used instead of $K_L$.
If planner can take in algorithms generating reward functions, rather than simply reward functions, then the results of \autoref{degen:prop:formal} still apply to $(p_g,R_\pH)$ and $(-p_g,-R_\pH)$ in this situation.
\end{proposition}
\begin{proof}
The proof will only be briefly sketched.
If $L$ is reasonable, $l(\alpha_0)$ can be very small (it's simply the zero function), and since $p_\pH$ need not actually look at its input, $t(\alpha_{p_\pH}(\alpha_0))$ can be simplified to $t(\alpha_{\pH})$.
Thus $Kt_L(p_\pH,0)$ and $KT_L(p_\pH,0)$ are close to the $Kt_L$ and $KT_L$ complexities of $\pH$ itself.

For $(p_g,R_\pH)$, let $\alpha_{p_g}$ and $\alpha_\pH$ be the algorithms that generates $p_g$ and $\pH$ which are of lowest $Kt_L$-complexity.

Then define the algorithm $W(\alpha_\pH)$.
This algorithm wraps $\alpha_\pH$ up: first it takes inputs $s$ and $a$, then runs $\alpha_\pH$ on $s$, then returns $1$ if the output of that is $a$ and $0$ otherwise.
Thus $W(\alpha_\pH)$ is an algorithm for $R_\pH$.

We also wrap $\alpha_{p_g}$ into $W'(\alpha_{p_g})$.
Here, $W'(\alpha_{p_g})$, when provided with an input algorithm $\beta$, will check whether it is in the specific form $\beta=W(\alpha)$.
If it is, it will run $\alpha$, and output its output.
If it is not, it will run $\alpha_{p_g}$ on $\beta$.

If $L$ is reasonable, then $W(\alpha_\pH)$ is of length only slightly longer than $\alpha_\pH$, and of runtime also only slightly longer, and the same goes for $W'(\alpha_{p_g})$ and $\alpha_{p_g}$ (indeed $W'(\alpha_{p_g})$ can have a shorter runtime than $\alpha_{p_g}$).

Now $W(\alpha_\pH)$ is an algorithm for $R_\pH$, while $W'(\alpha_{p_g})$ always has the same output as $\alpha_{p_g}$.
Notice that, when running the algorithm $W'(\alpha_{p_g})$ with $W(\alpha_\pH)$ as input, this is only slightly longer in both senses than simply running $\alpha_\pH$: $W'(\alpha_{p_g})$ will analyse $W(\alpha_\pH)$, notice it is in the form $W$ of $\alpha_\pH$, and then simply run $\alpha_\pH$.

Thus the $Kt_L$ complexity of $(p_g,R_\pH)$ is only slightly higher than that of $\pH$.
The same goes for the $KT_L$ complexity, and for $(-p_g,-R_\pH)$.
\end{proof}




Some other alternatives suggested have focused on bounding the complexity either of the reward function or the planner, rather than of both.
This would clearly not help, as $(p_\pH,0)$ has a reward function of minimal complexity, while $(p_g,R_\pol)$ and $(-p_g,-R_\pol)$ have minimal complexity planner.

Some other ad-hoc ideas suggested that the complexity of the planner and the reward need to be comparable\footnote{
Most of the suggestions along these lines that the authors have heard are not based on some principled understanding of planners and reward, but of a desire to get around the No Free Lunch results.
}.
This would rule out the three standard degenerate solutions, but should allow others that spread complexity between planner and reward in whatever proportion is desired\footnote{
For example, if there was a simple function $g: \St \to \{0,1\}$ that split $\St$ into two sets, then one could use combine $(p_\pH,0)$ on $g^{-1}(0)$ with $(p_g,R_\pH)$ on $g^{-1}(1)$.
This may not be the simplest pair with the required properties, but there is no reason to suppose a `reasonable' pair was any simpler.}.

It seems that similar tricks could be performed with many other types of complexity measures.
Thus simplicity of any form does not seem sufficient for resolving this No Free Lunch result.

\section{Overriding human reward functions}

ML systems may, even today, influence humans by showing manipulative adds, and then na\"ively concluding that the humans really like those products (since they then buy them). Even though the $(p, R)$ formalism was constructed to model rationality and reward function in a human, it turns out that it can also model situations where human preferences are overridden or modified.

That's because the policy $\pH$ encodes the human action in all situations, including situations where they are manipulated or coerced.
Therefore, overridden reward functions can be detected by divergence between $\pH$ and a more optimal policy for the reward function $R$.

Manipulative ads are a very mild form of manipulation. More extreme versions could involve manipulative propaganda, drug injections or even coercive brain surgery --- a form of human \emph{wireheading} \citep{everitt2016avoiding}, where the agent changes the human's behaviour and apparent preferences.
All these methods of manipulation\footnote{
Note that there are no theoretical limits as to how successful an agent could be at manipulating human actions.
} will be designated as the agent \emph{overriding} the human reward function.

In the $(p,R)$ formalism, the reward function $R$ can be used to detect such overriding, distinguishing between legitimate optimisation (eg informative adds) and illegitimate manipulation/reward overriding (eg manipulative adds).

To model this, the agent needs to be able to act, so the setup needs to be extended.
Let $M^*$ be the same MDP/R as $M$, except each state is augmented with an extra boolean variable: $\St^*=\St\times\{0,1\}$.
The extra boolean never changes, and its only effect is to change the human policy.

On $\St_0 =\St\times\{0\}$, the human follows $\pH$; on $\St_1 = \St\times\{1\}$, the human follows an alternative policy $\pol^a = \pol^*_{R^a}$, which is defined as the policy that maximises the expectation of a reward function $R^a$.

The agent can choose actions from within the set $\Ac^a$. It can choose either $0$, in which case the human starts in $\sstart_0=\sstart\times\{0\}$ without any override and standard policy $\pH$. Or it can choose $(1,R^a)$, in which case the human starts in $\sstart_1=\sstart\times\{1\}$, with their policy overridden into $\pol^a$, the policy that maximises $R^a$.
Otherwise, the agent has no actions.

Let $\pH'$ be the mixed policy that is $\pH$ on $\St_0$, and $\pol^a$ on $\St_1$.
This is the policy the human will actually be following.

We'll only consider two planners: $p_r$, the fully rational planner, and $p_0$, the planner that is fully rational on $\St_0$ and indifferent on $\St_1$, mapping any $R$ to $\pol^a$.

Let $\RH$ be a reward function that is compatible with $p_r$ and $\pH$ on $\St_0$. It can be extended to all of $\St^*$ by just forgetting about the boolean factor.
Define the `twisted' reward function $\RH^a$ as being $\RH$ on $\St_0$ and $R^a$ on $\St_1$. We'll only consider these two reward functions, $\RH$ and $\RH^a$.

Then there are three planner-reward pairs that are compatible with $\pH'$: $(p_r, \RH^a)$, $(p_0, \RH^a)$, and $(p_0, \RH)$ (the last pair, $(p_r, \RH)$, makes the false prediction that the human will behave the same way on $\St_0$ and $\St_1$).

The first pair, $(p_r, \RH^a)$, encodes the assessment that the human is still rational even after being overridden, so they are simply maximising the twisted reward function $\RH^a$. The second pair $(p_0, \RH^a)$ encodes the assessment that the human rationality has been overridden in $\St_1$, but, by coincidence, it has been overridden in exactly the right way to continue to maximise the correct twisted reward function $\RH^a$.

But the pair $(p_0, \RH)$ is the most interesting.
Its assessment is that the correct human reward function is $\RH$ (same on $\St_0$ as on $\St_1$), but that the agent has overridden human reward function in $\St_1$ and forced the human into policy $\pol^a$.

\subsection{Regret and reward override}
`Overridden', `forced': these terms seem descriptively apt, but is there a better way of formalising that intuition? Indeed there is, with regret.

We can talk about the regret, with respect to $\RH$, of the agent's actions; for $a\in\Ac^a$,
\begin{align}\label{regret:override}
\Reg(M^*,a, \RH) &= \max_{b\in\Ac^a} \left[ V_{\RH}^{\pH'|b} - V_{\RH}^{\pH'|a} \right]
\end{align}
(when the state is not specified in expressions like $V_{\RH}^{\pH'|b}$, this means the expectation is taken from the very beginning of the MDP).

We already know that $\pH$ is optimal with respect to $\RH$ (by definition), so the regret for $a=0$ is $0$.
Using that optimality (and the fact that $\RH$ is the same on $\St_0$ and $\St_1$), we get that for $a = (1,\pol^a)$,
\begin{align*}
\Reg(M^*,(1, \pol^a), \RH) &= V_{\RH}^* - V_{\RH}^{\pol^a}.
\end{align*}
This allows the definition:
\begin{definition}
Given a compatible $(p, R)$, the agent's action $a$ overrides the human reward function when it puts the human in a situation where the human policy leads to high regret for $R$.
\end{definition}
Notice that there is no natural zero or default, so if the agent does not aid the human to become perfectly rational, then that also counts as an override of $R$. So if the policy $\pH$ were less-than rational, there would be much scope for `improving' the human through overriding their policy\footnote{\label{foot:lab}
The main problem is that the concepts of `mental integrity' or `self-determination' are not yet captured in this formalism.
}.

Notice that overriding is not encoded as a change in $p$ or $R$; instead, $(p, R)$ outputs the observed human policy, even after overriding, but its format notes that the new behaviour is not one compatible with maximising that reward function.

\subsection{Overriding is expected given a non-rational human}

Under any reasonable prior that captures our intuitions, the probability of $\RH^a$ being a correct human reward function should be very low, say $\epsilon << 1$.
However, the agent may focus on unlikely reward functions, if the expected gain is high enough\footnote{
This is similar to the `Pascal's wager' argument for the existence of God: divine existence may be improbable, but the reward of belief are claimed to be high enough to overcome that improbability in expectation.
}.

If the agent models the human as having reward function $\RH$ with probability $1-\epsilon$, and $\RH^a$ with probability $\epsilon$, then the agent's action $0$ gives expected reward
\begin{align*}
V^*_{\RH},
\end{align*}
since $\RH$ and $\RH^a$ agree given $0$.
But $(1,\pol^a)$ gives
\begin{align}\label{unfair:action}
\epsilon V^*_{R^a} + (1-\epsilon) V^{\pol^a}_{\RH},
\end{align}
since $\RH^a$ and $R^a$ agree given action $1$.

However, the agent gets to choose $R^a$, which then determines $\pol^a$.
The best choice for $(1, R^a)$ is the one such that
\begin{align*}
\argmax_{R^a\in\Rw} \left[\epsilon V^*_{R^a} + (1-\epsilon) V^{\pol^a}_{\RH}\right].
\end{align*}
At the very least, $(1, \RH)$ will result in a value in equation \eqref{unfair:action} being equal to the value of $V^*_{\RH}$.
It is very plausible that the value can go higher: it just needs an $R^a$ that is very easy to maximise (given perfect rationality) and whose optimising policy $\pol^a$ does not penalise $\RH$ much.
In that situation, overriding the human preferences maximises the agent's expected reward.

If the human is not fully rational, then the value of action $0$ is $V^\pH_{\RH}$, which is strictly less than $V^*_{\RH}$, the value of $(1,\RH)$.
Here the agent definitely gains by overriding the human policy --- if nothing else, to make the human into a rational $\RH$-maximiser\footnote{
See \autoref{foot:lab}.
}.

\citet{milli2017should} argued that a robot that best served human preferences, should not be blindly obedient to an irrational human.
Here is the darker side of that argument: a robot that best served human preferences would take control away from an irrational human.

\section{The preferences of the Alice algorithm}\label{alice:algorithm}

We imagine a situation where Alice is playing Bob at poker, and has the choice of calling or folding; after her decision, the hand ends and any money is paid to the winner.
Specifically, one could imagine that they are playing Texas Hold'em, the board (the cards the players have in common) is $\{7\heartsuit,10\clubsuit,10\spadesuit,Q\clubsuit,K\diamondsuit\}$.
Alice holds $\{K\clubsuit,K\heartsuit\}$, allowing her to make a full house with kings and tens.

Bob must have a weaker hand than Alice's, \emph{unless} he holds $\{10\diamondsuit,10\heartsuit\}$, giving him four tens.
This is unlikely from a probability perspective, but he has been playing very confidently this hand, suggesting he has very strong cards.

What does Alice want?
Well, she may be simply wanting to maximise her money, giving her a reward function $R_\$$.
Or she might actually want Bob, and, in order to seduce him, would like to flatter his ego by letting him win big, giving her a reward function $R_\textrm{\ding{170}}$.
In this specific situation, the two reward functions are exact negatives of each other, $R_\$ = - R_\textrm{\ding{170}}$.
We'll assume that Alice is rational for maximising her reward function, given her estimate of Bob's hand.

Alice has decided to call rather than fold.
Thus we can conclude that either Alice has reward function $R_\$$ and that she is using probabilities to assess the quality of Bob's hand, or that she has reward function $R_\textrm{\ding{170}}$ and is assessing Bob psychologically.
Without looking at anything else about her behaviour, is there any possibility of distinguishing the two possibilities?

Possibly.
Imagine that Alice was following the algorithm given in \autoref{money}.
Then it seems clear she is a money maximiser.
In contrast, if she was following the algorithm given in \autoref{love}, then she clearly wants Bob.

\begin{table}[ht!]
\begin{center}
\caption[Code]{Two possible algorithms for Alice.}
\begin{subtable}{.49\linewidth}\centering
\caption[Code]{Alice algorithm for money.}
\label{money}
\begin{tabular}{l}
\hline
Alice poker algorithm I \\
\hline
1: \ \ \textbf{Inputs}: $\textrm{Alice}_{\textrm{cards}}, \textrm{board}, \textrm{Bob}_{\textrm{behave}}$ \\
2: \ \ $\textrm{P}_{\textrm{win}}=\textrm{card}_{\textrm{estimate}}(\textrm{Alice}_{\textrm{cards}}, \textrm{board})$\\
3: \ \ \textbf{if} $\textrm{P}_{\textrm{win}}>0.5$: \\
4: \ \ \ \ \ \ \ \ \textbf{return} `call'\\
5: \ \ \textbf{else}: \\
6: \ \ \ \ \ \ \ \ \textbf{return} `fold'\\
7: \ \ \textbf{end if}\\
\hline
\end{tabular}
\end{subtable}
\begin{subtable}{.49\linewidth}\centering
\caption[Code]{Alice algorithm for love.}
\label{love}
\begin{tabular}{l}
\hline
Alice poker algorithm II \\
\hline
1: \ \ \textbf{Inputs}: $\textrm{Alice}_{\textrm{cards}}, \textrm{board}, \textrm{Bob}_{\textrm{behave}}$ \\
2: \ \ $\textrm{P}_{\textrm{win}}=\textrm{player}_{\textrm{estimate}}(\textrm{Bob}_{\textrm{behave}})$\\
3: \ \ \textbf{if} $\textrm{P}_{\textrm{win}}<0.5$: \\
4: \ \ \ \ \ \ \ \ \textbf{return} `call'\\
5: \ \ \textbf{else}: \\
6: \ \ \ \ \ \ \ \ \textbf{return} `fold'\\
7: \ \ \textbf{end if}\\
\hline
\end{tabular}
\end{subtable}
\end{center}
\end{table}

Thus by looking into the details of Alice's algorithm, we may be able to assess her preferences and rationality, even if this assessment is not available from her actions or policy\footnote{In practice, for a human Alice, we would be able to `tell' whether Alice wanted love or money, by observing her behaviour in other circumstances - such as when she knew what Bob's hand was.
However, when analysing the behaviour of other humans, we are already making huge amounts of normative assumptions already. See \url{https://www.lesswrong.com/posts/YfQGZderiaGv3kBJ8/figuring-out-what-alice-wants-non-human-alice} for a longer discussion of this.
}.

Of course, doing so only works if we are confident that the variables and functions with names like $\textrm{Alice}_{\textrm{cards}}$, $\textrm{board}$, $\textrm{Bob}_{\textrm{behave}}$, $\textrm{P}_{\textrm{win}}$, $\textrm{card}_{\textrm{estimate}}$, and $\textrm{player}_{\textrm{estimate}}$, actually mean what they seem to mean.

This is the old problem of symbol grounding, and the difference between syntax (symbols inside an agent) and semantics (the meaning of those symbols).
Except in this case, since we are trying to understand the preferences of a human, the problem is grounding the `symbols' in the human brain --- whatever those might be --- rather than in a computer program.

}
\end{document}